\documentclass[11pt]{article}
\usepackage[left=1in,right=1in,top=0.8in,bottom=1in]{geometry}
\usepackage{amsmath}
\usepackage{amssymb}
\usepackage{amsthm}
\usepackage{graphicx}
\usepackage{float}
\usepackage{xcolor}
\usepackage{tikz-cd}
\usepackage{multicol}
\usepackage{enumitem}
\usepackage{natbib}

\newcommand{\bX}{\mathbf{X}}
\newcommand{\bA}{\mathbf{A}}
\newcommand{\bx}{\mathbf{x}}
\newcommand{\bB}{\mathbf{B}}
\newcommand{\bW}{\mathbf{W}}

\definecolor{editgreen}{RGB}{0,128,0}

\theoremstyle{definition}
\newtheorem{remark}{Remark}
\newtheorem{theorem}{Theorem}
\newtheorem{corollary}{Corollary}
\newcounter{algorithm}
\newcommand{\algorithmtitle}[2]{%
  \refstepcounter{algorithm}%
  \label{alg:#1}%
  \textbf{Algorithm \thealgorithm: #2}%
}
\newcommand{\algref}[1]{Algorithm~\ref{alg:#1}}
\newcommand{\includegraphicsorplaceholder}[2][]{%
  \IfFileExists{#2}{%
    \ifnum\pdffilesize{#2}>0
      \includegraphics[#1]{#2}%
    \else
      \fbox{\parbox[c][2in][c]{0.9\textwidth}{\centering Image file \texttt{#2} is empty.}}%
    \fi
  }{%
    \fbox{\parbox[c][2in][c]{0.9\textwidth}{\centering Image file \texttt{#2} was not found.}}%
  }%
}

\title{Learned Look-Ahead Splitting Rule for CART}
\author{
  \begin{minipage}[t]{0.33\textwidth}
    \centering
    Andrew Gao\\[0.3em]
    \scriptsize Thomas Jefferson High School\\ for Science and Technology,\\ Alexandria, VA 22312, USA
  \end{minipage}
  \hfill
  \begin{minipage}[t]{0.3\textwidth}
    \centering
    Tianlin Liu, Ruichen Han\\[0.3em]
    \scriptsize Lynbrook High School,\\ San Jose, CA 95129, USA
  \end{minipage}
  \hfill
  \begin{minipage}[t]{0.33\textwidth}
    \centering
    Lu Tian\thanks{Corresponding author.}\\[0.3em]
    \scriptsize Department of Biomedical Data Science,\\ Stanford University,\\ Stanford, CA 94305, USA
  \end{minipage}
}

\date{}

\begin{document}

\maketitle

\begin{abstract}
Classification and regression trees are typically constructed using a greedy splitting rule that maximizes the immediate reduction in prediction error at each node. Although this strategy is computationally efficient, it can miss splits that yield small short-term gains but create substantial downstream improvements after further partitioning. We propose a look-ahead tree-building method that evaluates each candidate split by the prediction error reduction achieved after growing a conventional CART subtree below that split. Because the full look-ahead procedure can be computationally expensive, we also describe a smart look-ahead algorithm that learns downstream split values using node-level features. The proposed framework preserves the interpretability of recursive partitioning while improving split selection in hierarchical or interaction-driven settings. We conduct a simulation study comparing conventional, full look-ahead, and smart look-ahead methods under several settings and apply the proposed methods to analyze two real data examples demonstrating the merit of the new methods.
\end{abstract}

\section{Introduction}
Classification and regression trees (CART), introduced by Breiman, Friedman, Olshen, and Stone, are widely used because they predict through sequences of simple, interpretable decision rules \citep{breiman1984}. Their standard construction is greedy: each split is selected by its immediate reduction in prediction error, without considering structure revealed deeper in the tree. A locally weak but necessary split may therefore be rejected even if it would support a substantially better subtree. Subsequent pruning can remove branches from a grown tree but cannot introduce a promising split that was never selected.

This limitation is particularly important in the presence of deep interactions, where a variable with a weak marginal effect may not be selected near the root even though it enables substantial improvements through later splits. Consider the XOR-type regression \(Y=I(X_1X_2>0)\) with independent, symmetric predictors. Splitting either predictor at zero produces no population-level immediate reduction in squared error because both child nodes have mean response \(1/2\). Yet splitting the resulting child nodes at zero on the other predictor perfectly separates the four response regions in the absence of noise.

This observation motivates evaluating a candidate split by the quality of the subtree constructed beneath it rather than solely by its immediate gain. The resulting look-ahead procedure preserves CART's recursive structure and interpretability while crediting early splits for improvements realized several levels later. Because constructing a complete downstream subtree for every candidate split can be computationally expensive, we also develop a learned approximation to the look-ahead value.

Prior decision-tree research has considered one-level look-ahead for classification, which can sometimes produce larger or less accurate trees \citep{murthy1995}; efficient information-gain-based look-ahead for numerical predictors \citep{elomaa2003}; and anytime ID3 algorithms using bounded-depth search or stochastic downstream construction \citep{esmeir2004}. Our full procedure follows the same broad principle of evaluating splits through their descendants but directly compares binary variable--cutoff pairs for a continuous outcome using the squared-error risk of downstream conventional CART subtrees. Our primary extension, smart look-ahead, learns this value from bootstrap-generated node features and replaces repeated online subtree construction with a reusable prediction model. Reinforcement learning trees similarly select variables according to potential future improvement, but require a suitable initial model of the association between $\bX$ and $Y$ to guide tree construction \citep{zhu2015}.

In summary, this paper makes three main contributions: 
\begin{enumerate}
\item we formulate a full look-ahead CART procedure for continuous outcomes in which each candidate split is evaluated by growing conventional CART subtrees below the forced split; 
\item we develop a computationally cheaper refinement that learns downstream split values from a prediction model trained on bootstrap CART trees;
\item we demonstrate that the full look-ahead CART procedure can achieve lower training error than the conventional procedure.
\end{enumerate}
We evaluate the new approach in an extensive simulation study that compares conventional, full look-ahead, and smart look-ahead CART in terms of predictive performance and computational cost.

\subsection{Review of Regular CART}

CART is a tree-based statistical learning method for predicting an outcome \(Y\) from covariates $\bX$ \citep{breiman1984}. Specifically, it is a binary decision tree that recursively partitions the covariate space into regions that are increasingly homogeneous with respect to the outcome. Each region, or leaf, is defined by the intersection of a set of binary decision rules, and the predicted outcome is constant within each leaf. Thus, CART defines a piecewise constant function \(f_{\mathcal T}(\bx): \Omega_{\bX} \to \Omega_Y\), where the subscript \(\mathcal T\) denotes the particular tree and \(\Omega_{\bX}\) and \(\Omega_Y\) are the supports of \(\bX\) and \(Y\), respectively. We focus on a continuous outcome $Y$.

CARTs are commonly grown using a simple greedy, recursive splitting rule. In the regression setting, suppose the training data consist of \(n\) independent and identically distributed (i.i.d.) observations
\[
\{(Y_i,\bX_i): i=1,\ldots,n\}.
\]
At any leaf (node) \(\bA\) of the current tree, the conventional CART algorithm considers a collection of candidate binary splits \(s \in \mathcal{S}(\bA)\). Each split partitions \(\bA\) into two child nodes, denoted by \(\bA_L(s)\) and \(\bA_R(s)\). The usual splitting criterion chooses the split that maximizes the immediate reduction in total squared prediction error,
\begin{align*}
\Delta_{\mathrm{GD}}(\bA, s)
&= \sum_{\bX_i\in \bA}(Y_i-\widehat{\mu}_{\bA})^2-\sum_{\bX_i\in \bA_L(s)}(Y_i-\widehat{\mu}_{\bA_L(s)})^2-\sum_{\bX_i\in \bA_R(s)}(Y_i-\widehat{\mu}_{\bA_R(s)})^2\\
&= R(\bA) - R\{\bA_L(s)\} - R\{\bA_R(s)\},
\end{align*}
where $\widehat{\mu}_{\bB}$ is the sample mean of all outcomes for observations in node $\bB$. This regression tree is then grown recursively until a stopping criterion is met, such as a maximum tree depth, a minimum number of observations in terminal leaves, or a minimum required improvement in total squared prediction error. The procedure is summarized in Algorithm \ref{alg:convention}. 

\algorithmtitle{regular-cart}{Regular CART for a continuous outcome}\label{alg:convention}
\begin{enumerate}[label=\arabic*.]
    \item Start with the root node containing all training observations.
    \item At a current leaf $\bA$, check whether a stopping condition has been satisfied; if so, declare $\bA$ terminal. Otherwise, form the set of candidate binary splits $\mathcal{S}(\bA)$.
    \item For each candidate split $s\in\mathcal{S}(\bA)$, compute the greedy improvement $\Delta_{\mathrm{GD}}(\bA, s)$.
    \item Select the leaf $\widehat{\bA}^{\mathrm{GD}}$ and split $\widehat{s}^{\mathrm{GD}} \in \mathcal{S}(\widehat{\bA}^{\mathrm{GD}})$ as
    \[
    \left(\widehat{\bA}^{\mathrm{GD}}, \widehat{s}^{\mathrm{GD}}\right)=\arg\max_{\bA,\,s\in\mathcal{S}(\bA)}\Delta_{\mathrm{GD}}(\bA, s).
    \]
    \item Split $\widehat{\bA}^{\mathrm{GD}}$ using $\widehat{s}^{\mathrm{GD}}$ and repeat steps 2--5 until all leaves are terminal.
\end{enumerate}

After stopping, this procedure yields a large tree ${\cal T}_F,$ which may overfit the data. A pruning procedure may then be conducted to simplify the tree and improve out-of-sample prediction \citep{breiman1984}. This overfit-then-prune strategy can outperform a tree that is stopped early using only the greedy splitting rule, because pruning evaluates candidate subtrees after the larger tree structure has already been explored. 

In general, pruning generates a nested sequence of subtrees by repeatedly removing the branch rooted at an internal node according to a selected criterion. Specifically, for a current tree ${\cal T}_k$, consider every internal node $\mathbf{B}$ whose branch contains more than one terminal leaf. Let ${\cal T}_{k,\mathbf{B}}$ be the subtree of ${\cal T}_k$ rooted at $\mathbf{B}$. The average increase in empirical risk per terminal leaf removed is
\[
\alpha(\mathbf{B})
=
\frac{R(\mathbf{B})-R({\cal T}_{k,\mathbf{B}})}{|\mathcal{L}({\cal T}_{k,\mathbf{B}})|-1},
\]
where 
\[
R({\cal T})
=
\sum_{\ell\in \mathcal{L}({\cal T})}
\sum_{\bX_i\in \ell}(Y_i-\mu_\ell)^2
\] 
is the empirical risk of tree ${\cal T}$, $\mathcal{L}({\cal T})$ is its set of terminal leaves, and $|\mathcal{L}({\cal T})|$ is the number of terminal leaves. The branch with the smallest value of $\alpha(\mathbf{B})$ is pruned. If several branches attain the same minimum, they may be pruned simultaneously, or one branch may be selected according to a prespecified tie-breaking rule. Repeating this operation gives a nested sequence of trees starting from ${\cal T}_F$:
\begin{equation}
{\cal T}_F={\cal T}_0 \supset {\cal T}_1 \supset \cdots \supset {\cal T}_K,
\label{eq:prune-tree-seq}
\end{equation}
from which the final CART can be selected. For example, one may choose a final tree ${\cal T}_{\widehat{k}}$ with the smallest $\mathrm{BIC}({\cal T}_k)$ \citep{schwarz1978}, where the BIC for tree ${\cal T}_k$ under a Gaussian model is
\[
\mathrm{BIC}({\cal T}_k)
= n\log\left\{\frac{R({\cal T}_k)}{n}\right\}
+ \deg({\cal T}_k)\log n,
\]
where $\deg({\cal T})$ is the degree of freedom of tree ${\cal T}$.  In this paper, we set it to be the total number of nodes.

\begin{remark}
At additional computational cost, cross-validation can also be used to select the final tree from sequence (\ref{eq:prune-tree-seq}) \citep{stone1974}. Tree ${\cal T}_k$ is associated with a critical value $\alpha_k$, the value of $\alpha(\mathbf{B}_k)$ for the branch pruned from ${\cal T}_{k-1}$ to obtain ${\cal T}_k$. In $V$-fold cross-validation, the data are randomly partitioned into $V$ approximately equal-sized folds, ${\cal I}_1,\ldots,{\cal I}_V$. For each fold $v$, grow a large tree ${\cal T}_{F}^{(-v)}$ using all observations except those in fold $v$, and construct its pruning path
\[
{\cal T}^{(-v)}_F={\cal T}_0^{(-v)} \supset {\cal T}_1^{(-v)} \supset \cdots \supset {\cal T}_{K_v}^{(-v)},
\]
indexed by critical values
\[
0=\alpha_0^{(-v)} \le \alpha_1^{(-v)} \le \cdots \le \alpha_{K_v}^{(-v)}.
\]
The held-out mean squared prediction error of ${\cal T}_k^{(-v)}$, denoted by $\widehat{CV}_v\left(\alpha_k^{(-v)}\right)$, is calculated using observations from ${\cal I}_v$. Define the piecewise constant risk function $\widehat{CV}_v(\alpha)=\widehat{CV}_v\left(\alpha_k^{(-v)}\right)$ for any $\alpha \in [\widehat{\alpha}_k^{(-v)},\widehat{\alpha}_{k+1}^{(-v)})$. The cross-validated risk estimate is
\[
\widehat{\mathrm{CV}}(\alpha)
=
\frac{1}{V}
\sum_{v=1}^V
\widehat{CV}_v(\alpha).
\]
One then sets
\[
\widehat\alpha=\arg\min_{\alpha} \widehat{\mathrm{CV}}(\alpha),
\]
and choose the subtree corresponding to $\widehat\alpha$ as the final CART.
\end{remark}

\section{Look-Ahead Method} 

Although the greedy strategy is simple and computationally efficient, it tends to be ``nearsighted.'' A split may produce only a small immediate reduction in squared prediction error yet create child nodes whose internal structure allows substantial downstream improvements after additional splits. Conversely, a split with a large immediate improvement may lead to child nodes that cannot be refined effectively. The conventional CART rule can therefore miss splits that are valuable primarily because they enable better future partitioning of the covariate space. This limitation motivates a look-ahead approach that evaluates a split not only by its immediate gain but also by the predictive improvement it enables downstream. Specifically, we propose a look-ahead splitting method that evaluates each candidate split by the improvement achieved by the full subtree grown after that split. The look-ahead procedure retains Steps 1, 2, and 5 of Algorithm \ref{alg:convention} and replaces only Steps 3--4 with the look-ahead evaluation and selection rules below.

% \textcolor{red} {We define the depth of a node as the number of edges from the root, so the root node has depth \(0\), its children have depth \(1\), and so forth. For a candidate split at a current node, the look-ahead depth is the number of additional conventional CART splitting levels grown below the two child nodes created by the forced candidate split. Thus, look-ahead depth \(0\) evaluates only the immediate improvement from the candidate split and is equivalent to the greedy conventional CART splitting criterion, whereas look-ahead depth \(1\) permits one additional splitting level below the child nodes.}

\algorithmtitle{look-ahead-cart}{Look-ahead CART for a continuous outcome}
\begin{enumerate}[label=\arabic*.]
    \item Start with the root node containing all training observations.
    \item At a current leaf \(\bA\), check whether a stopping condition has been satisfied; if so, declare \(\bA\) terminal. Otherwise, form the set of candidate binary splits \(\mathcal{S}(\bA)\).
    \item For each candidate split \(s\in\mathcal{S}(\bA)\):
    \begin{enumerate}[label*=\arabic*.]
        \item Force the split of \(\bA\) into two child nodes.
        \item Grow conventional CART subtrees on the left and right child nodes separately. 
        \item Let \(\mathcal{T}_{\bA,s}\) denote the resulting subtree rooted at \(\bA\). Compute the reduction in total squared prediction error as the look-ahead value:
        \[
        \Delta_{\mathrm{LA}}(\bA,s)
        =R(\bA)-R(\mathcal T_{\bA,s}).
        \]
    \end{enumerate}
    \item Select the leaf \(\widehat{\bA}^{\mathrm{LA}}\) and split \(\widehat{s}^{\mathrm{LA}}\in \mathcal{S}(\widehat{\bA}^{\mathrm{LA}})\) as
    \[
    (\widehat{\bA}^{\mathrm{LA}}, \widehat{s}^{\mathrm{LA}})=\arg\max_{\bA, s\in\mathcal{S}(\bA)}\Delta_{\mathrm{LA}}(\bA, s).
    \]
    \item Split \(\widehat{\bA}^{\mathrm{LA}}\) using \(\widehat{s}^{\mathrm{LA}}\) and repeat Steps 2--5 until all leaves are terminal.
\end{enumerate}

The look-ahead method retains the nested structure and interpretability of CART but replaces the one-step splitting criterion with the prediction error improvement attainable by the downstream subtree induced by each candidate split. The main advantage of this approach is that it directly targets the long-run value of a split within the tree-building process. It gives higher priority to splits that create useful future partitioning opportunities, even when their immediate reduction in prediction error is modest. This can be especially beneficial when the signal is hierarchical or interaction-driven.

The main limitation of the look-ahead method is its computational burden. At each node, the method requires growing two conventional CARTs for every candidate split. When the number of predictors, cutpoints, or observations is large, this nested tree-building procedure can be substantially more expensive than the standard greedy algorithm. Several practical strategies can reduce the computational cost: (1) adopting stopping rules that terminate the growth of ${\cal T}_{\bA,s}$ early, such as restricting its maximum depth, and (2) limiting the number of candidate splits per predictor.

\section{Smart Look-Ahead Method} 

The look-ahead method in the previous section can be computationally expensive. An appealing solution is to estimate the downstream improvement associated with a candidate split $(\bA,s)$ without explicitly growing a CART subtree after the split. Specifically, one can predict $\Delta_{\mathrm{LA}}(\bA,s)$ using features derived from $(\bA,s)$. This idea motivates the smart look-ahead method.

Let $\bB$ be a node, and let ${\cal T}_{\bB}$ denote the subtree rooted at $\bB$ and generated using the conventional greedy method.
Suppose that we can train a learner $\widehat{G}(\cdot)$ to predict the proportion of variation explained by this subtree:
\[\widehat{G}(\bW_{\bB}) \approx 1-\frac{R(\mathcal{T}_{\bB})}{R(\bB)},\]
where $\bW_{\bB}$ is a feature vector constructed from the observations $\left\{(\bX_i,Y_i)\mid \bX_i\in \bB\right\}$.
For a leaf $\bA$ and a split $s\in {\cal S}(\bA)$, let $\bA_{L(s)}$ and $\bA_{R(s)}$ denote the resulting left and right child nodes, respectively. We can then approximate $\Delta_{\mathrm{LA}}(\bA,s)$ by
\begin{equation}  \widehat{\Delta}_{\mathrm{LA}}(\bA, s)=
        R(\bA)-\left\{1-\widehat{G}(\bW_{\bA_{L(s)}})\right\}R(\bA_{L(s)})-\left\{1-\widehat{G}(\bW_{\bA_{R(s)}})\right\}R(\bA_{R(s)}).
\label{eq:approxdelta} 
\end{equation}
Therefore, Steps 3--4 of \algref{look-ahead-cart} may be replaced by the following steps:

\begin{itemize}
  \item[3.] Calculate $\widehat{\Delta}_{\mathrm{LA}}(\bA,s)$ using (\ref{eq:approxdelta}).
  \item[4.]   Select the leaf \(\widehat{\bA}^{\mathrm{SLA}}\) and split \(\widehat{s}^{\mathrm{SLA}}\in \mathcal{S}(\widehat{\bA}^{\mathrm{SLA}})\) as
    \[
    (\widehat{\bA}^{\mathrm{SLA}}, \widehat{s}^{\mathrm{SLA}})=\arg\max_{\bA, s\in\mathcal{S}(\bA)}\widehat{\Delta}_{\mathrm{LA}}(\bA, s).
    \]
\end{itemize}

The remaining task is to train the learner $\widehat{G}(\cdot)$. To this end, we bootstrap the observed data and construct a conventional CART for each bootstrap sample. Denote the resulting tree for the $b$th bootstrap sample by ${\cal T}^b$. For each node $\bA$ of ${\cal T}^b$, we identify the subtree rooted at $\bA$, denoted by ${\cal T}^b_{\bA}$. We then calculate the proportion of variation explained by ${\cal T}^b_{\bA}$, denoted by $O^b_{\bA}$, and extract a vector of node-specific predictive features, denoted by $\bW^b_{\bA}$. Thus, the $b$th bootstrap sample yields the training set
\[
\mathbf{D}^b=\{(\bW^b_{\bA},O^b_{\bA}):\bA\mbox{ is a node of } {\cal T}^b\}.
\]
By repeating this procedure for $B$ bootstrap samples, we obtain the combined training set
\[
\mathbf{D}=\mathbf{D}^1\cup \mathbf{D}^2\cup\cdots\cup \mathbf{D}^B
\]
with a sufficient number of observations. Using these data, any suitable machine-learning method may be used to train $\widehat{G}(\cdot)$ to predict $O^b_{\bA}$ from $\bW^b_{\bA}$. For example, random forests \citep{breiman2001} or gradient boosting methods \citep{friedman2001} may be used. The predictive features play an important role in estimating the proportion of variation explained by the CART subtree. We construct $\bW^b_{\bA}$ using the sample size and depth of node $\bA$ within ${\cal T}^b$, the correlations between $Y$ and the components of $\bX$, and the proportions of variation explained by depth-1 and depth-2 CARTs constructed using all predictors within $\bA$ via the conventional greedy algorithm. This process is summarized in Algorithm \ref{alg:learning}.

\algorithmtitle{train-learned-value}{Training the learned look-ahead value model} \label{alg:learning}

\begin{enumerate}
\item For $b=1,\ldots,B$:
\begin{itemize}
\item Draw a bootstrap sample from the original training data.
\item Grow a conventional CART ${\cal T}^b$ using the bootstrap sample.
\item For each node \(\bA\) of \(\mathcal{T}^b\):
    \begin{enumerate}
    \item Compute the feature vector $\bW^b_{\bA}$.
    \item Compute the proportion of variation explained by the CART subtree rooted at node $\bA$, denoted by $O^b_{\bA}$.
    \end{enumerate}
\item Collect the resulting supervised learning data set
\[
\{(\bW^b_{\bA}, O^b_{\bA}):\bA \mbox{ is a node of } {\cal T}^b \}.
\]
\end{itemize}
\item Assemble all training data 
\[
\{(\bW^b_{\bA}, O^b_{\bA}):\bA \mbox{ is a node of } {\cal T}^b,\ b=1,\ldots,B\}.
\]
\item Train a prediction model \(\widehat{G}(\cdot)\) to predict \(O^b_{\bA}\) from \(\bW^b_{\bA}\) using the assembled training data.
\end{enumerate}

The effectiveness of the smart look-ahead algorithm depends on whether the auxiliary model can accurately learn the downstream improvements based on node characteristics.

The speed of the smart algorithm depends on the balance between its offline training cost and online tree-construction cost. The full look-ahead method is expensive because it requires growing two separate CARTs for every candidate split at each leaf of the current tree. The smart look-ahead method avoids this repeated subtree growth by replacing it with evaluations of the trained prediction model \(\widehat{G}(\cdot)\). Once the model has been trained, evaluating \(\widehat{\Delta}_{\mathrm{LA}}(\bA,s)\) for each candidate split is typically much faster than growing complete downstream CARTs. Therefore, the smart algorithm can substantially reduce the cost of building the final tree, especially when there are many candidate splits. However, this speed gain requires an additional training stage in which bootstrap CARTs are grown, node-level features are extracted, downstream improvements are computed, and the auxiliary prediction model is fitted. For a single analysis of a small data set, this upfront cost may outweigh the savings during final tree construction. Thus, the smart method shifts computation from repeated online subtree growth to a one-time offline learning step and is particularly appealing for growing large trees with extensive split searches.

\section{Theoretical Properties}

This section establishes a finite-sample comparison between conventional CART and the exact full look-ahead procedure. The comparison concerns empirical risk on the training observations and does not imply an ordering of out-of-sample prediction risks.

We consider tree growth without pruning. Both methods use the same candidate-split sets, minimum leaf size, maximum depth, stopping rules, and deterministic tie-breaking rule. We further assume that the stopping rules are local to each node and that every exact look-ahead evaluation grows the candidate tree to completion using the conventional CART method.

For a partially grown tree $\mathcal T$, let $\mathcal C(\mathcal T)$ denote the set of admissible actions $a=(\bA,s)$, where $\bA$ is a terminal leaf and $s\in\mathcal S(\bA)$. Let $\mathcal T_a$ denote the tree obtained by applying only split $s$ to leaf $\bA$. Define the deterministic conventional splitting policy by
\[
\pi_{\mathrm{GD}}(\mathcal T)
\in
\arg\max_{a\in\mathcal C(\mathcal T)}
\Delta_{\mathrm{GD}}(a).
\]
Let $C_{\mathrm{GD}}(\mathcal T)$ be the terminal tree obtained by repeatedly applying $\pi_{\mathrm{GD}}$ to $\mathcal T$, and define its completion risk by
\[
V_{\mathrm{GD}}(\mathcal T)
=
R\{C_{\mathrm{GD}}(\mathcal T)\}.
\]
The exact full look-ahead policy evaluates each candidate action by its conventional completion risk and selects
\begin{equation}
\pi_{\mathrm{LA}}(\mathcal T)
\in
\arg\min_{a\in\mathcal C(\mathcal T)}
V_{\mathrm{GD}}(\mathcal T_a).
\label{eq:lookahead-policy}
\end{equation}
Under local stopping rules, completing the unchanged leaves contributes the same quantity for every candidate action. Therefore, minimizing the criterion in (\ref{eq:lookahead-policy}) is equivalent to maximizing the local look-ahead improvement $\Delta_{\mathrm{LA}}(\bA,s)$ defined in Section 2.

\begin{theorem}
Let $\mathcal T_0$ be the common initial tree. Let
\[
\mathcal T^{\mathrm{GD}}=C_{\mathrm{GD}}(\mathcal T_0)
\]
be the terminal tree constructed by conventional CART, and let $\mathcal T^{\mathrm{LA}}$ be the terminal tree obtained by repeatedly applying the policy in (\ref{eq:lookahead-policy}). Then
\[
R(\mathcal T^{\mathrm{LA}})
\leq
R(\mathcal T^{\mathrm{GD}}).
\]
Consequently, the training MSE of the exact full look-ahead tree is no greater than that of the conventional tree.
\end{theorem}

\begin{proof}
Consider any nonterminal tree $\mathcal T$. Let
\[
a_{\mathrm{GD}}=\pi_{\mathrm{GD}}(\mathcal T)
\qquad\text{and}\qquad
a_{\mathrm{LA}}=\pi_{\mathrm{LA}}(\mathcal T).
\]
Because conventional completion from $\mathcal T$ begins with $a_{\mathrm{GD}}$,
\begin{equation}
V_{\mathrm{GD}}(\mathcal T)
=
V_{\mathrm{GD}}(\mathcal T_{a_{\mathrm{GD}}}).
\label{eq:greedy-completion}
\end{equation}
Moreover, $a_{\mathrm{GD}}\in\mathcal C(\mathcal T)$ is one of the actions considered by the look-ahead policy. Hence, by (\ref{eq:lookahead-policy}),
\begin{align}
V_{\mathrm{GD}}(\mathcal T_{a_{\mathrm{LA}}})
\leq
V_{\mathrm{GD}}(\mathcal T_{a_{\mathrm{GD}}})
=
V_{\mathrm{GD}}(\mathcal T),
\label{eq:policy-improvement}
\end{align}
where the equality follows from (\ref{eq:greedy-completion}). Thus, $V_{\mathrm{GD}}$ is nonincreasing after every split selected by the look-ahead policy.

Let $\mathcal T^{\mathrm{LA}}_0,\mathcal T^{\mathrm{LA}}_1,\ldots,\mathcal T^{\mathrm{LA}}_M$ denote the sequence of partial trees generated by the look-ahead policy, with $\mathcal T^{\mathrm{LA}}_0=\mathcal T_0$ and $\mathcal T^{\mathrm{LA}}_M=\mathcal T^{\mathrm{LA}}$ terminal. Repeated application of (\ref{eq:policy-improvement}) gives
\[
V_{\mathrm{GD}}(\mathcal T^{\mathrm{LA}})
\leq
V_{\mathrm{GD}}(\mathcal T_0).
\]
Because $\mathcal T^{\mathrm{LA}}$ is terminal, $C_{\mathrm{GD}}(\mathcal T^{\mathrm{LA}})=\mathcal T^{\mathrm{LA}}$, so the left-hand side equals $R(\mathcal T^{\mathrm{LA}})$. By definition, the right-hand side equals $R(\mathcal T^{\mathrm{GD}})$. This completes the proof.
\end{proof}

The preceding theorem applies only to the exact full look-ahead procedure under the stated assumptions and does not generally extend to limited-depth look-ahead, pruning, or out-of-sample prediction risk. In particular, the smart look-ahead method replaces the exact look-ahead improvement with an estimated value and may therefore select a suboptimal action. Nevertheless, its departure from the exact procedure can be controlled when the estimated improvement is uniformly close to the exact improvement over all admissible actions. If the approximation error is at most $\epsilon$ at each splitting step, the action selected by smart look-ahead has exact improvement at most $2\epsilon$ below that of the exact look-ahead action. The following corollary accumulates this one-step loss over the $M$ splits used to construct the terminal smart look-ahead tree, yielding a finite-sample training-risk bound relative to conventional CART.

\begin{corollary}[Finite-sample guarantee for smart look-ahead CART]
Let $\mathcal T_0$ be the common initial tree, and let
\[
\mathcal T^{\mathrm{GD}} = C_{\mathrm{GD}}(\mathcal T_0)
\]
denote the terminal tree constructed by conventional CART. Suppose the smart
look-ahead procedure generates the sequence
\[
\mathcal T^{\mathrm{SLA}}_0 = \mathcal T_0,\quad
\mathcal T^{\mathrm{SLA}}_1,\quad \ldots,\quad
\mathcal T^{\mathrm{SLA}}_M = \mathcal T^{\mathrm{SLA}},
\]
where $\mathcal T^{\mathrm{SLA}}$ is terminal. At a nonterminal tree $\mathcal T$, for each admissible action
$a=(\bA,s)\in \mathcal C(\mathcal T)$, let
\[
\Delta_{\mathrm{LA}}(a)
=
\Delta_{\mathrm{LA}}(\bA,s)
\]
denote the exact look-ahead improvement, and let
$\widehat{\Delta}_{\mathrm{LA}}(a)$ denote its smart look-ahead approximation.
Suppose that, 
\[
\max_{0\le m <M}\sup_{a\in \mathcal C(\mathcal T_m^{\mathrm{SLA}})}
\left|
\widehat{\Delta}_{\mathrm{LA}}(a)
-
\Delta_{\mathrm{LA}}(a)
\right|
\leq \epsilon.
\tag{6}
\]

Under the same
assumptions on candidate splits, stopping rules, and deterministic tie breaking
as in Theorem~1,
\[
R(\mathcal T^{\mathrm{SLA}})
\leq
R(\mathcal T^{\mathrm{GD}})
+
2M\epsilon.
\tag{7}
\]
\end{corollary}

\begin{proof}
Fix a step $m$ and write $\mathcal T=\mathcal T_m^{\mathrm{SLA}}$. Let
$a_m^*\in\arg\max_{a\in\mathcal C(\mathcal T)}\Delta_{\mathrm{LA}}(a)$.
By the uniform approximation bound in~(6) and the definition of $\widehat a_m$,
\[
\Delta_{\mathrm{LA}}(a_m^*)
\leq \widehat{\Delta}_{\mathrm{LA}}(a_m^*)+\epsilon
\leq \widehat{\Delta}_{\mathrm{LA}}(\widehat a_m)+\epsilon
\leq \Delta_{\mathrm{LA}}(\widehat a_m)+2\epsilon.
\]
Under the local stopping rules, maximizing $\Delta_{\mathrm{LA}}(a)$ is equivalent to minimizing $V_{\mathrm{GD}}(\mathcal T_a)$. Since the conventional action $a_{\mathrm{GD}}=\pi_{\mathrm{GD}}(\mathcal T)$ is admissible,
\[
V_{\mathrm{GD}}(\mathcal T_{\widehat a_m})
\leq V_{\mathrm{GD}}(\mathcal T_{a_m^*})+2\epsilon
\leq V_{\mathrm{GD}}(\mathcal T_{a_{\mathrm{GD}}})+2\epsilon
=V_{\mathrm{GD}}(\mathcal T)+2\epsilon.
\]
Applying this inequality over the $M$ smart look-ahead steps gives
\[
V_{\mathrm{GD}}(\mathcal T_M^{\mathrm{SLA}})
\leq V_{\mathrm{GD}}(\mathcal T_0)+2M\epsilon.
\]
Finally, $\mathcal T_M^{\mathrm{SLA}}=\mathcal T^{\mathrm{SLA}}$ is terminal, so
$V_{\mathrm{GD}}(\mathcal T_M^{\mathrm{SLA}})=R(\mathcal T^{\mathrm{SLA}})$, while
$V_{\mathrm{GD}}(\mathcal T_0)=R(\mathcal T^{\mathrm{GD}})$. Substitution proves~(7).
\end{proof}
\section{Simulation Study}

We conduct a simulation study to compare three tree-building strategies:
conventional greedy CART, full look-ahead CART, and smart look-ahead CART using
learned split values. The primary objective is to evaluate whether look-ahead
splitting improves prediction accuracy in settings where the optimal early split
has limited immediate predictive value but enables substantial downstream
improvement. A second objective is to determine whether the smart look-ahead
method can approximate the predictive performance of full look-ahead CART while
reducing computational cost. The simulation study therefore evaluates both
statistical performance and computational efficiency.

Let $\bX=(X_1,\ldots,X_{10})$ denote the predictors, with each component generated independently from a uniform distribution $U(-1, 1)$. The continuous outcome is generated from
\[
Y=f(\bX)+\varepsilon,
\]
where $f(\bX)$ is the true regression function and $\varepsilon\sim N(0,0.25)$ is a mean-zero noise term. Specifically, we consider the following four simulation settings:
\begin{enumerate}
\item $f(\bX)=I(X_1X_2>0)+I(X_3X_4>0)+I(X_5X_6>0),$ 
\item $f(\bX)=I(X_1X_2>0)+I(X_3X_4>0)+(X_5+X_6)/8,$ 
\item $f(\bX)=I(X_1X_2>0)+(X_3+X_4+X_5+X_6)/8,$ 
\item $f(\bX)=(X_1+X_2+X_3+X_4+X_5+X_6)/8.$ 
\end{enumerate}

In all four settings, $X_7$, $X_8$, $X_9$, and $X_{10}$ are noise features.
For each simulation setting, we generate independent training sets with $n\in\{500,1000,2000\}$ observations. For each training set, we implement eight CART procedures: conventional, full look-ahead, and two versions of smart look-ahead, each with and without BIC-based pruning. When growing each CART, we set the maximum depth to 6 and the minimum node size to 25. For all look-ahead methods, we consider 20 candidate split points per predictor, chosen as empirical quantiles. To implement the smart look-ahead methods, we repeatedly bootstrap the training set to obtain at least 500 nodes for training a random forest with 50 regression trees as $\widehat{G}(\cdot)$.
The first smart look-ahead method uses node size, node depth (i.e., the number of binary rules defining the node), and the correlations between $Y$ and the components of $\bX$ as predictive features. The second smart look-ahead method additionally uses the proportions of variation explained by depth-1 and depth-2 CARTs trained within the node using the conventional greedy algorithm. Pruning is performed using the BIC criterion, with the total number of internal and terminal nodes treated as the degrees of freedom to encourage simple trees. Each fitted regression tree is evaluated using mean squared prediction error (MSE) on an independently generated test set of $m=10{,}000$ observations. The number of terminal leaves is recorded as a measure of tree complexity. For the smart look-ahead methods, we also record the coefficient of determination, $R^2$, of $\widehat{G}(\cdot)$.

We repeat the training and evaluation procedure 400 times for each simulation setting and summarize the training and test MSEs and the number of terminal leaves using their means  across repetitions. The results for $n=1000$ are summarized in Table \ref{tab:simulation-results-n1000-rerun} and Figure \ref{fig:sim-mse}. Results for $n=500$ and $n=2000$ are provided in Tables \ref{tab:simulation-results-n500} and \ref{tab:simulation-results-n2000-official} of supplementary materials, respectively. In the first three settings, which contain important interactions, the look-ahead methods performed substantially better than the conventional method. Pruning did not materially change this pattern but helped stabilize predictive performance (Figure \ref{fig:sim-mse}). The second smart look-ahead method, which uses a richer set of predictive features, performed similarly to or even better than the full look-ahead method. The first smart look-ahead method performed worse than the full look-ahead method but still outperformed the conventional method. Its poorer performance was also reflected in the lower $R^2$ of the auxiliary random forest model. However, the second smart look-ahead method had a higher computational cost than the first. In Setting 4, all methods performed similarly, with the conventional method performing slightly better than the others. This result is expected because the data-generating model contains no interactions and therefore provides no particular advantage to look-ahead splitting.

\begin{figure}[htbp] 
\centering
\includegraphicsorplaceholder[width=\textwidth]{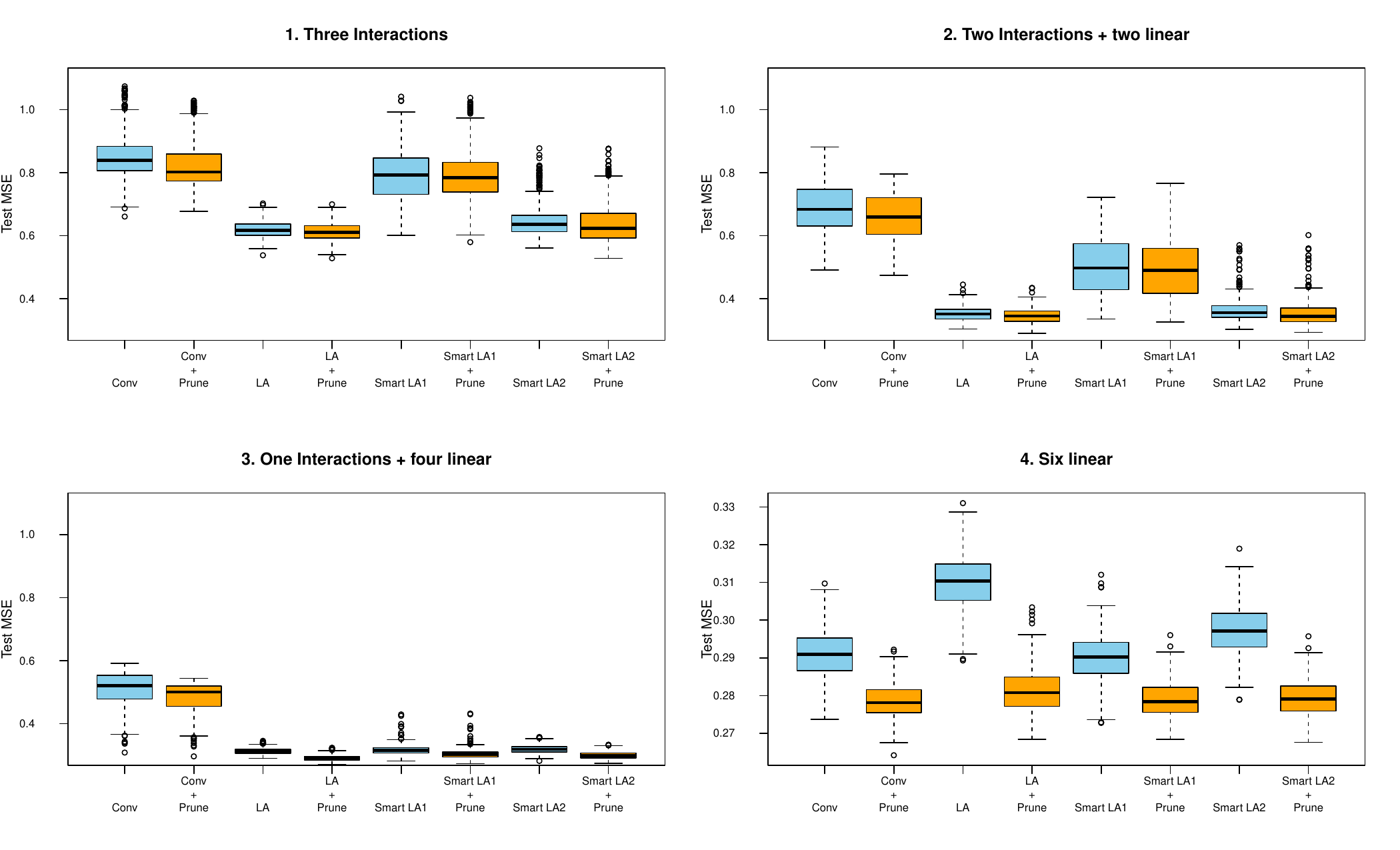}
\caption{Distributions of test-set MSE under four simulation settings for $n=1000$. Conv: conventional greedy method; LA: full look-ahead method; Smart LA 1: the first smart look-ahead method with the first set of node level features; Smart LA 2: the second smart look-ahead method with the second set of node level features; Blue: without pruning; Orange: with BIC-based pruning.}
\label{fig:sim-mse}
\end{figure}

\begin{table}[htbp]

\centering
\caption{Summary of the independently rerun simulation results across the four settings for $n=1000$. Conventional: conventional greedy method; Full LA: full look-ahead method; Smart LA 1: the first smart look-ahead method with the first set of node level features; Smart LA 2: the second smart look-ahead method with the second set of node level features.   Entries are based on  400 simulation repetitions (Training MSE: mean squared prediction error in the training set; Test MSE: mean squared prediction error in the test set; No. of leaves: mean number of terminal leaves of constructed CARTs; RF $R^2:$ the proportion of variation explained by $\hat{G}(\cdot)$ in predicting the value of a candidate split).\\}
\label{tab:simulation-results-n1000-rerun}
\small
\setlength{\tabcolsep}{5pt}
\resizebox{\textwidth}{!}{%
\begin{tabular}{llcccc}
\hline
Setting & Method & Training MSE & Test MSE & No. of leaves & RF $R^2$ \\
\hline
1 & Conventional & 0.704 & 0.853 & 22.518 & -- \\
  & Conventional + pruning & 0.773 & 0.824 & 8.973 & -- \\
  & Full LA & 0.458 & 0.619 & 31.553 & -- \\
  & Full LA + pruning & 0.472 & \textbf{0.613} & 27.265 & -- \\
  & Smart LA 1 & 0.670 & 0.795 & 24.648 & 0.772 \\
  & Smart LA 1 + pruning & 0.736 & 0.791 & 11.060 & 0.772 \\
  & Smart LA 2 & 0.517 & 0.649 & 28.925 & 0.943 \\
  & Smart LA 2 + pruning & 0.551 & 0.640 & 19.265 & 0.943 \\
\hline
2 & Conventional & 0.556 & 0.687 & 23.700 & -- \\
  & Conventional + pruning & 0.618 & 0.661 & 8.875 & -- \\
  & Full LA & 0.277 & 0.352 & 30.510 & -- \\
  & Full LA + pruning & 0.288 & \textbf{0.346} & 24.563 & -- \\
  & Smart LA 1 & 0.420 & 0.501 & 26.328 & 0.728 \\
  & Smart LA 1 + pruning & 0.450 & 0.492 & 15.610 & 0.728 \\
  & Smart LA 2 & 0.296 & 0.366 & 29.778 & 0.930 \\
  & Smart LA 2 + pruning & 0.317 & 0.355 & 18.448 & 0.930 \\
\hline
3 & Conventional & 0.403 & 0.509 & 24.358 & -- \\
  & Conventional + pruning & 0.455 & 0.483 & 6.520 & -- \\
  & Full LA & 0.223 & 0.313 & 31.198 & -- \\
  & Full LA + pruning & 0.260 & \textbf{0.291} & 12.763 & -- \\
  & Smart LA 1 & 0.257 & 0.317 & 26.738 & 0.689 \\
  & Smart LA 1 + pruning & 0.290 & 0.304 & 6.853 & 0.689 \\
  & Smart LA 2 & 0.244 & 0.319 & 27.988 & 0.905 \\
  & Smart LA 2 + pruning & 0.286 & 0.299 & 5.873 & 0.905 \\
\hline
4 & Conventional & 0.224 & 0.291 & 25.743 & -- \\
  & Conventional + pruning & 0.259 & \textbf{0.278} & 4.665 & -- \\
  & Full LA & 0.207 & 0.310 & 31.705 & -- \\
  & Full LA + pruning & 0.266 & 0.281 & 3.998 & -- \\
  & Smart LA 1 & 0.228 & 0.290 & 25.695 & 0.860 \\
  & Smart LA 1 + pruning & 0.265 & 0.279 & 3.833 & 0.860 \\
  & Smart LA 2 & 0.221 & 0.297 & 27.960 & 0.958 \\
  & Smart LA 2 + pruning & 0.266 & 0.279 & 4.015 & 0.958 \\
\hline
\end{tabular}
}
\end{table}

In a second set of simulations, we compare the computational speeds of the full and smart look-ahead methods. We generate 50 training data sets, each containing $n=1000$ observations, under Setting 4. For each data set, we construct CART using both methods while varying the number of candidate cutoffs per predictor from 10 to 100 in increments of 10 and record the running time. Figure \ref{fig:sim-time} summarizes the running times across the 50 repetitions. When the number of candidate cutoffs is small, the full look-ahead method is faster because the cost of training the auxiliary prediction model for the smart look-ahead method exceeds the computational savings from using it to score a relatively small number of candidate splits. As the number of candidate cutoffs increases, however, the repeated subtree fitting required by the full look-ahead method becomes increasingly expensive, and the smart look-ahead method becomes substantially faster.

\begin{figure}[H]
\centering
\includegraphicsorplaceholder[width=\textwidth]{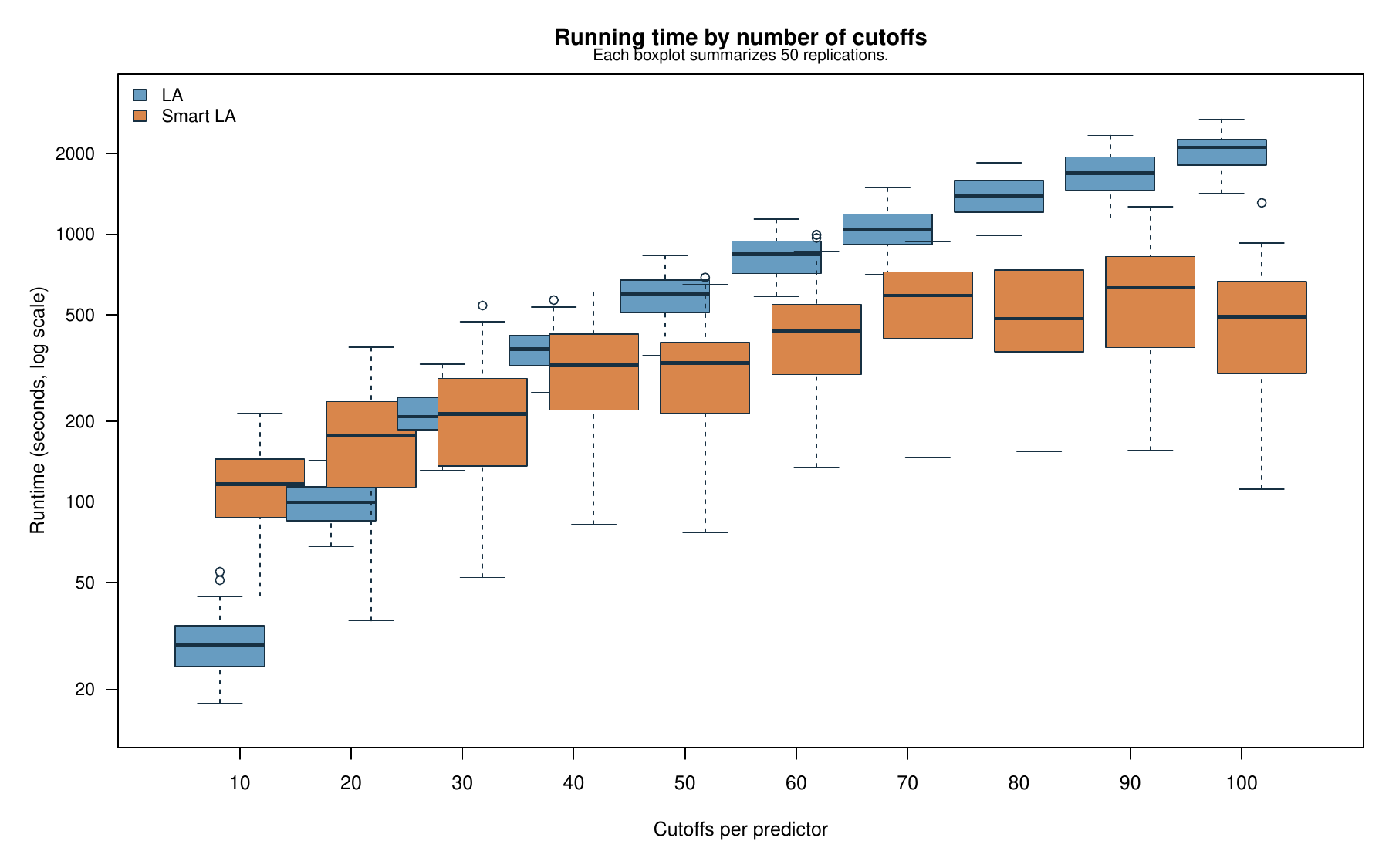}
\caption{Comparison of running times for the full and smart look-ahead methods across 50 replicates.}
\label{fig:sim-time}
\end{figure}

\section{Example}
To evaluate the proposed methods beyond selected simulation settings, we apply them to analyze two real-world datasets and compare their predictive performance with that of conventional CART. The analysis considers full look-ahead and both smart look-ahead procedures, each with and without BIC-based pruning. The model performance is estimated via repeated $80/20$ random train-test splitting as the true data generating mechanism is unknown.

\subsection{Parkinsons Telemonitoring}

We evaluated our methods on the Parkinsons Telemonitoring dataset from the UCI Machine Learning Repository \citep{tsanas2009dataset}. The dataset contains biomedical voice measurements collected from $42$ individuals with early-stage Parkinson’s disease during a six-month telemonitoring trial. It consists of 5,875 observations and includes two target variables, motor UPDRS and total UPDRS, as well as subject age, gender, time since baseline recruitment, and 16 biomedical voice measures including five jitter measures, six shimmer measures, noise-to-harmonics ratio,  harmonics-to-noise ratio, two nonlinear measures, and pitch period entropy.

For each target variable, we conducted our analysis at the observation level to evaluate the same eight CART procedures considered in the simulation study: conventional CART, full look-ahead, and two versions of smart look-ahead methods, each with and without BIC-based pruning. We performed $100$ independent evaluations. In each evaluation, we randomly assigned $80\%$ of the observations to the training set and the remaining $20\%$ to the test set following the  approach used  in the  paper \citep{tsanas2010}. Subject ID and time since baseline recruitment were excluded from the predictor variables. We considered 10 candidate splits per predictor.

In each bootstrap iteration used to generate supervised data for the auxiliary model in smart look-ahead method, we drew a sample with replacement equal in size to the full training set. We grew a conventional CART tree using each bootstrap sample and included in the auxiliary-model training set not only the nodes retained in the final bootstrap tree but also all nodes evaluated during tree construction. Given the size and complexity of this data set, this approach produced a larger and more varied collection of node-level training examples. We repeated the procedure until we had collected at least 10,000 node-level examples and then fitted the auxiliary random forest, $\widehat{G}(\cdot)$, using 500 regression trees. The results are summarized in Figure~\ref{fig:parkinsons-boxplots-observation-level} and Table~\ref{tab:parkinsons-results-observation-level}.

\begin{figure}[H]
\centering
\begin{tabular}{@{}cc@{}}
\includegraphics[width=0.48\textwidth]{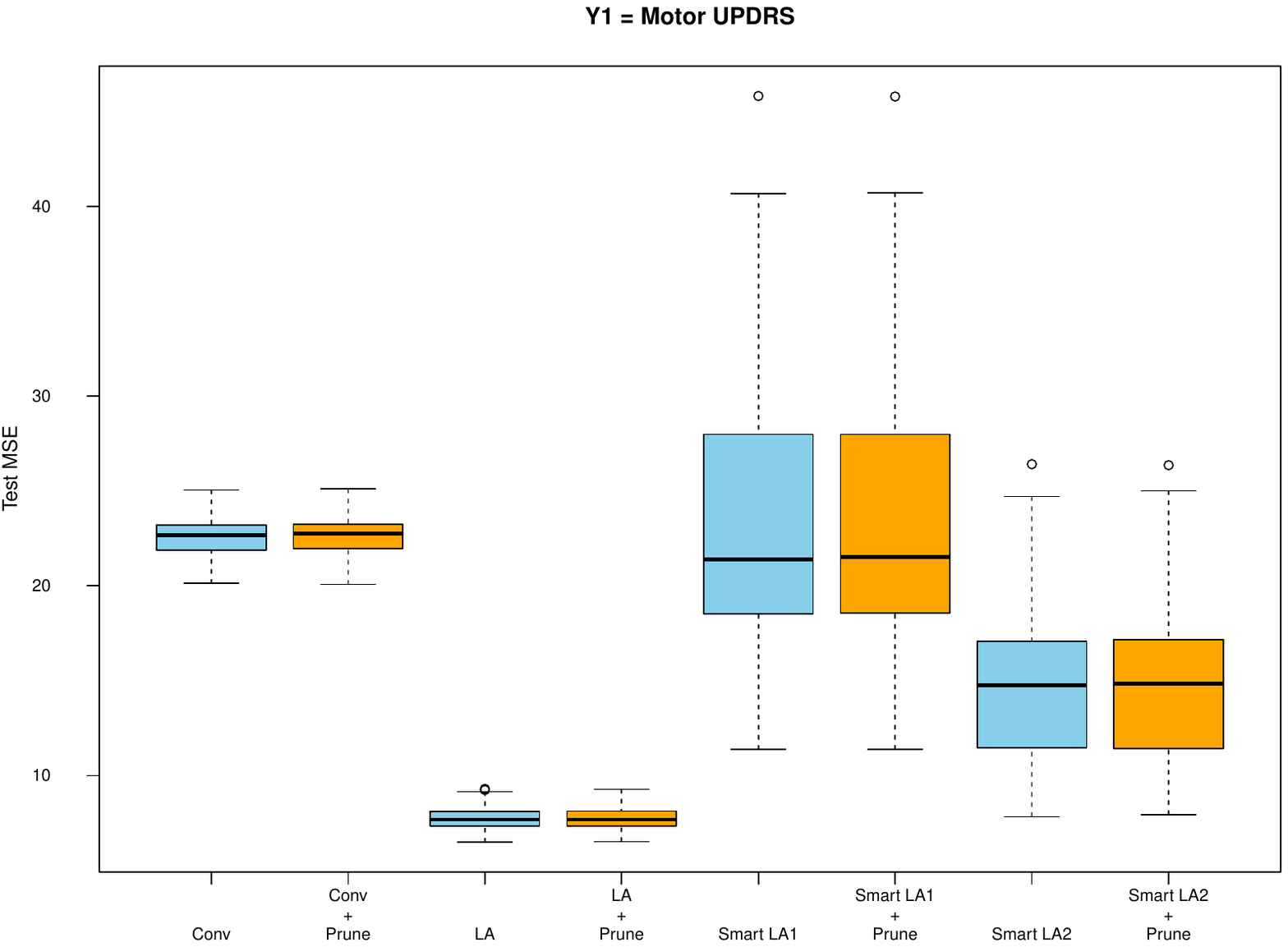} &
\includegraphics[width=0.48\textwidth]{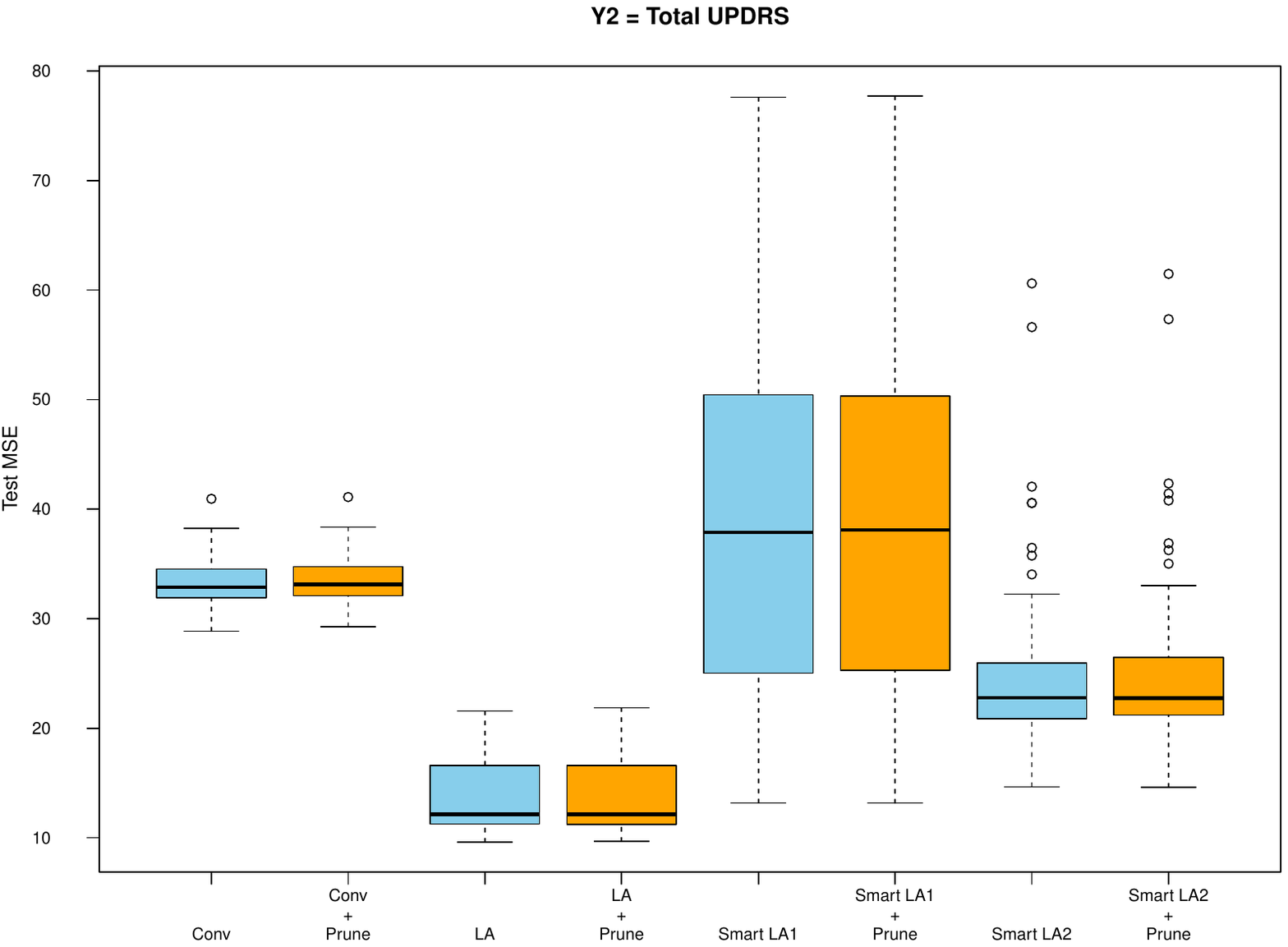} \\
(a) Motor UPDRS & (b) Total UPDRS
\end{tabular}
\caption{Distributions of test MSE across the 100 independent observation-level train--test partitions for the Parkinsons Telemonitoring dataset. Conv: conventional greedy method; LA: full look-ahead method; Smart LA 1: the first smart look-ahead method with the first set of node level features; Smart LA 2: the second smart look-ahead method with the second set of node level features; Blue: without pruning; Orange: with BIC-based pruning. Panel (a) shows results for predicting motor UPDRS, and panel (b) shows results for predicting total UPDRS.}
\label{fig:parkinsons-boxplots-observation-level}
\end{figure}

\begin{table}[htbp]
\centering
\caption{Summary of the Parkinsons Telemonitoring results. Entries are based on 100 independent observation-level train--test partitions. Conventional: conventional greedy method; Full LA: full look-ahead method; Smart LA 1: the first smart look-ahead method with the first set of node level features; Smart LA 2: the second smart look-ahead method with the second set of node level features.\\}
\label{tab:parkinsons-results-observation-level}
\small
\setlength{\tabcolsep}{4pt}
\resizebox{\textwidth}{!}{%
\begin{tabular}{llcccc}
\hline
Target Variable & Method & Training MSE & Test MSE & No. leaves & RF \(R^2\) \\
\hline
Motor UPDRS & Conventional & 21.603 & 22.658 & 42.210 & -- \\
& Conventional + pruning & 21.857 & 22.695 & 21.940 & -- \\
& Full LA & 7.121 & \textbf{7.735} & 53.480 & -- \\
& Full LA + pruning & 7.260 & 7.754 & 40.590 & -- \\
& Smart LA1 & 22.716 & 23.294 & 40.160 & 0.972 \\
& Smart LA1 + pruning & 22.893 & 23.327 & 24.240 & 0.972 \\
& Smart LA2 & 14.316 & 14.947 & 52.500 & 0.995 \\
& Smart LA2 + pruning & 14.537 & 15.007 & 34.690 & 0.995 \\
\hline
Total UPDRS & Conventional & 32.377 & 33.311 & 44.250 & -- \\
& Conventional + pruning & 32.803 & 33.552 & 27.790 & -- \\
& Full LA & 12.414 & \textbf{13.810} & 61.610 & -- \\
& Full LA + pruning & 12.689 & 13.833 & 48.270 & -- \\
& Smart LA1 & 38.006 & 38.838 & 43.340 & 0.963 \\
& Smart LA1 + pruning & 38.349 & 38.973 & 28.100 & 0.963 \\
& Smart LA2 & 23.112 & 24.212 & 54.180 & 0.992 \\
& Smart LA2 + pruning & 23.497 & 24.421 & 36.910 & 0.992 \\
\hline
\end{tabular}
}
\end{table}

In this observation-level analysis, full look-ahead achieved the lowest mean test MSE for both motor and total UPDRS. The second smart look-ahead method did not match the performance of full look-ahead, but it achieved substantially lower mean test MSE than conventional CART for both outcomes. The first smart look-ahead method performed similarly to conventional CART for motor UPDRS but had  a  higher mean test MSE for total UPDRS. This contrast suggests that the simpler auxiliary model used by the first smart look-ahead method may not adequately capture the more complex interaction structure associated with total UPDRS, whereas the richer features used by the second smart look-ahead method may provide a more accurate approximation of downstream split values. Pruning had little effect on the mean test MSE for any of the methods.

Lastly, using  all observations, we fitted final unpruned regression trees from full look-ahead and the second smart look-ahead methods for each target variable. The resulting trees for motor UPDRS are shown in Figures~\ref{fig:parkinsons-y1-full-la-tree} and~\ref{fig:parkinsons-y1-smart-la2-tree}, and the corresponding trees for total UPDRS are shown in Figures~\ref{fig:parkinsons-y2-full-la-tree} and~\ref{fig:parkinsons-y2-smart-la2-tree} in the supplementary materials.

\begin{remark}
Observation-level cross-validation ignores the fact that each patient contributes multiple observations and may therefore give overly optimistic results when the goal is to evaluate prediction performance in a new group of patients. We also conducted patient-level cross-validation, in which all observations from a subset of patients were held out for estimating test error. The resulting test error was much larger than the corresponding training error, suggesting that the constructed CART model is more suitable for predicting updated UPDRS values for existing patients than for predicting UPDRS values in new patients.
\end{remark}

\subsection{Concrete Compressive Strength} 
We evaluated our methods on the Concrete Compressive Strength dataset from the UCI Machine Learning Repository \citep{yeh1998dataset}. The data set contains 1,030 observations, with concrete compressive strength as the target variable and eight predictors including cement, blast furnace slag,  fly ash, water, superplasticizer, coarse aggregate, fine aggregate, and age.

We evaluated the same eight CART procedures considered in the simulation study and first example: conventional CART, full look-ahead, and two versions of smart look-ahead methods, each with and without BIC-based pruning. We performed 100 independent evaluations. For each evaluation, we randomly assigned 80\% of the 1,030 observations to the training set and the remaining 20\% to the test set. We considered 10 candidate splits per predictor. To train the auxiliary random forest used by the smart look-ahead procedures, we drew observations with replacement from the training set in each bootstrap replicate and repeated this procedure until we had collected at least 1,000 nodes. The results are summarized in Figure \ref{fig:concrete-boxplot} and Table \ref{tab:concrete-results}.

\begin{figure}[H]
\centering
\includegraphicsorplaceholder[width=0.65\textwidth]{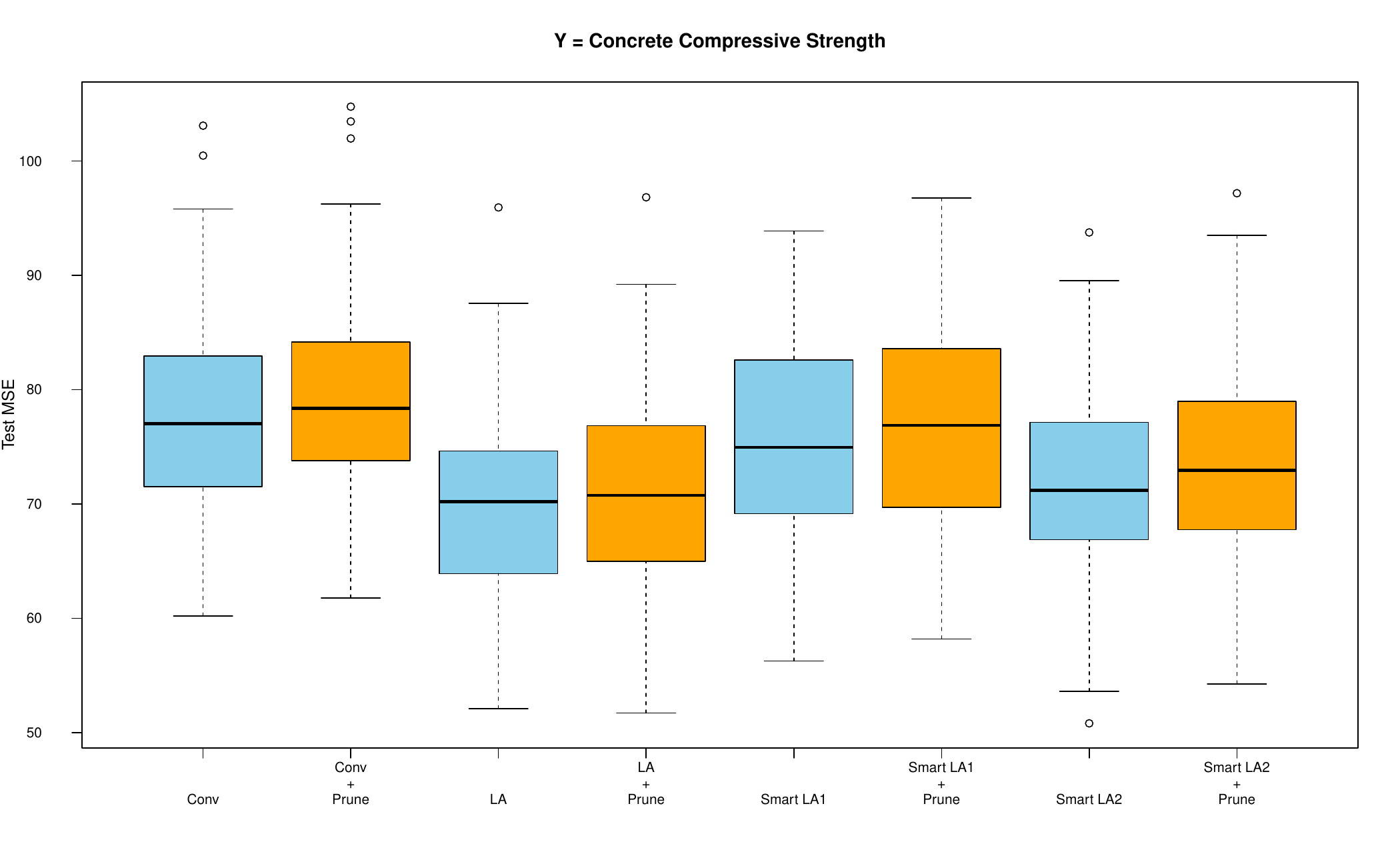}
\caption{Distributions of test MSE across the 100 independent train--test partitions for the Concrete Compressive Strength dataset. Conv: conventional greedy method; LA: full look-ahead method; Smart LA 1: the first smart look-ahead method with the first set of node level features; Smart LA 2: the second smart look-ahead method with the second set of node level features; Blue: without pruning; Orange: with BIC-based pruning.}
\label{fig:concrete-boxplot}
\end{figure}

\begin{table}[htbp]
\centering
\caption{Summary of the Concrete Compressive Strength results. Entries are means across 100 repetitions. Conventional: conventional greedy method; Full LA: full look-ahead method; Smart LA 1: the first smart look-ahead method with the first set of node level features; Smart LA 2: the second smart look-ahead method with the second set of node level features.\\
}
\label{tab:concrete-results}
\small
\setlength{\tabcolsep}{4pt}
\resizebox{\textwidth}{!}{%
\begin{tabular}{llcccc}
\hline
Target Variable & Method & Training MSE & Test MSE & No. leaves & RF \(R^2\) \\
\hline
Concrete Compressive Strength & Conventional & 64.186 & 77.756 & 22.970 & -- \\
& Conventional + pruning & 66.431 & 79.402 & 18.680 & -- \\
& Full LA & 51.787 & \textbf{69.677} & 26.350 & -- \\
& Full LA + pruning & 53.432 & 70.574 & 23.110 & -- \\
& Smart LA1 & 62.136 & 75.477 & 23.530 & 0.936 \\
& Smart LA1 + pruning & 64.262 & 76.955 & 19.630 & 0.936 \\
& Smart LA2 & 56.922 & 72.051 & 24.630 & 0.993 \\
& Smart LA2 + pruning & 58.560 & 73.338 & 21.580 & 0.993 \\
\hline
\end{tabular}
}
\end{table}

Full look-ahead achieved lower training and test MSE than conventional CART. The first smart look-ahead method yielded modest reductions in both measures, while the second smart look-ahead method yielded larger reductions and achieved performance closer to that of full look-ahead.

Using all observations, we then fitted two final trees: trees from full look-ahead and the second smart look-ahead methods, both without pruning. The resulting trees are shown in Figures \ref{fig:concrete-full-la-tree} and \ref{fig:concrete-smart-la2-tree}, respectively, in the supplementary section.

\section{Discussion}

This paper proposes a look-ahead approach that addresses the myopic nature of conventional CART splitting. The full method evaluates each candidate split by the quality of the downstream subtree that follows it, allowing early splits to receive credit for improvements revealed only after further partitioning. The smart method reduces the computational burden by learning downstream split values from node-level features. Our simulations suggest that look-ahead splitting can improve prediction in interaction-driven settings and that the smart approximation can retain much of this benefit when many candidate splits must be evaluated.

The framework is closely related to reinforcement learning and approximate dynamic programming \citep{sutton2018}. A partially grown tree can be viewed as a state, a candidate split as an action, and the reduction in prediction error as a reward. Conventional CART maximizes the immediate reward, whereas full look-ahead CART evaluates a split through a downstream rollout \citep{bertsekas1997}. The smart method resembles action-value approximation because it predicts the long-term value of a split without explicitly completing the rollout. Although the proposed procedure is not a full reinforcement-learning algorithm, this interpretation suggests potential extensions.
In particular, the downstream evaluator need not remain conventional CART. Future methods could grow the rollout trees using look-ahead or smart look-ahead splitting and iteratively update the splitting policy using the resulting values. Because recursive full look-ahead evaluation would be expensive, a practical strategy could use full look-ahead near the root and smart or conventional splitting at deeper nodes. Such policy-improvement approaches may yield stronger split evaluations while controlling computation.

The framework can also accommodate other outcomes by changing the node loss and subtree criterion. Classification trees could use impurity, deviance, or log-loss criteria \citep{breiman1984}; quantile trees could use check loss \citep{chaudhuri2002}; and survival trees could use censoring-aware prediction criteria \citep{leblanc1992}. Extensions to count outcomes, treatment-effect estimation \citep{athey2016}, and model-based recursive partitioning \citep{zeileis2008} are also possible. Important remaining questions include computational scalability, overfitting of the auxiliary value model, suitable pruning methods, out-of-sample guarantees, and the choice of rollout depth.

\clearpage
\section{Supplementary Tables and Figures}

% n = 500
\begin{table}[H]
\centering
\caption{Summary of the simulation results across the four settings for $n=500$. Entries are based on 400 simulation repetitions. Conventional: conventional greedy method; Full LA: full look-ahead method; Smart LA 1: the first smart look-ahead method with the first set of node level features; Smart LA 2: the second smart look-ahead method with the second set of node level features (Training MSE: mean squared prediction error in the training set; Test MSE: mean squared prediction error in the test set; No. of leaves: mean number of terminal leaves of constructed CARTs; RF $R^2:$ the proportion of variation explained by $\hat{G}(\cdot)$ in predicting the value of a candidate split).}
\label{tab:simulation-results-n500}
\small
\setlength{\tabcolsep}{5pt}
\resizebox{\textwidth}{!}{%
\begin{tabular}{llcccc}
\hline
Setting & Method & Training MSE & Test MSE & No. of leaves & RF $R^2$ \\
\hline
1 & Conventional            & 0.735 & 0.926 & 14.030 & -- \\
& Conventional + pruning    & 0.823 & 0.892 &  5.558 & -- \\
& Full LA                   & 0.512 & \textbf{0.662} & 16.188 & -- \\
& Full LA + pruning         & 0.520 & 0.668 & 15.270 & -- \\
& Smart LA1                 & 0.692 & 0.859 & 15.343 & 0.798 \\
& Smart LA1 + pruning       & 0.780 & 0.844 &  6.203 & 0.798 \\
& Smart LA2                 & 0.551 & 0.697 & 15.333 & 0.948 \\
& Smart LA2 + pruning       & 0.580 & 0.705 & 12.540 & 0.948 \\
\hline
2 & Conventional            & 0.574 & 0.731 & 14.133 & -- \\
& Conventional + pruning    & 0.651 & 0.698 &  4.988 & -- \\
& Full LA                   & 0.331 & 0.415 & 16.130 & -- \\
& Full LA + pruning         & 0.335 & 0.415 & 15.475 & -- \\
& Smart LA1                 & 0.473 & 0.590 & 15.355 & 0.769 \\
& Smart LA1 + pruning       & 0.519 & 0.571 &  7.758 & 0.769 \\
& Smart LA2                 & 0.341 & \textbf{0.414} & 15.338 & 0.938 \\
& Smart LA2 + pruning       & 0.350 & 0.415 & 14.083 & 0.938 \\
\hline
3 & Conventional            & 0.406 & 0.526 & 14.160 & -- \\
& Conventional + pruning    & 0.464 & 0.497 &  3.950 & -- \\
& Full LA                   & 0.231 & 0.317 & 16.380 & -- \\
& Full LA + pruning         & 0.259 & \textbf{0.301} &  9.025 & -- \\
& Smart LA1                 & 0.270 & 0.342 & 15.813 & 0.739 \\
& Smart LA1 + pruning       & 0.302 & 0.328 &  6.423 & 0.739 \\
& Smart LA2                 & 0.245 & 0.326 & 15.950 & 0.921 \\
& Smart LA2 + pruning       & 0.284 & 0.305 &  5.348 & 0.921 \\
\hline
4 & Conventional            & 0.224 & 0.296 & 14.580 & -- \\
& Conventional + pruning    & 0.260 & 0.283 &  3.320 & -- \\
& Full LA                   & 0.209 & 0.313 & 16.540 & -- \\
& Full LA + pruning         & 0.265 & 0.286 &  3.240 & -- \\
& Smart LA1                 & 0.228 & 0.296 & 15.138 & 0.856 \\
& Smart LA1 + pruning       & 0.268 & \textbf{0.282} &  2.578 & 0.856 \\
& Smart LA2                 & 0.223 & 0.301 & 14.928 & 0.957 \\
& Smart LA2 + pruning       & 0.271 & 0.283 &  2.338 & 0.957 \\
\hline
\end{tabular}
}
\end{table}

\clearpage

% n=2000, Andrew's version
\begin{table}[htbp]
\centering
\caption{Summary of the simulation results across the four settings for $n=2000$. Entries are based on 400 simulation repetitions. Conventional: conventional greedy method; Full LA: full look-ahead method; Smart LA 1: the first smart look-ahead method with the first set of node level features; Smart LA 2: the second smart look-ahead method with the second set of node level features (Training MSE: mean squared prediction error in the training set; Test MSE: mean squared prediction error in the test set; No. of leaves: mean number of terminal leaves of constructed CARTs; RF $R^2:$ the proportion of variation explained by $\hat{G}(\cdot)$ in predicting the value of a candidate split).}
\label{tab:simulation-results-n2000-official}
\small
\setlength{\tabcolsep}{5pt}
\resizebox{\textwidth}{!}{%
\begin{tabular}{llcccc}
\hline
Setting & Method & Training MSE & Test MSE & No. of leaves & RF $R^2$ \\
\hline
1 & Conventional            & 0.709 & 0.808 & 30.075 & -- \\
& Conventional + pruning    & 0.756 & 0.788 & 12.258 & -- \\
& Full LA                   & 0.378 & \textbf{0.468} & 56.585 & -- \\
& Full LA + pruning         & 0.383 & 0.471 & 53.885 & -- \\
& Smart LA1                 & 0.665 & 0.744 & 34.473 & 0.757 \\
& Smart LA1 + pruning       & 0.706 & 0.741 & 16.708 & 0.757 \\
& Smart LA2                 & 0.468 & 0.546 & 48.323 & 0.942 \\
& Smart LA2 + pruning       & 0.488 & 0.550 & 36.863 & 0.942 \\
\hline
2 & Conventional            & 0.567 & 0.665 & 34.633 & -- \\
& Conventional + pruning    & 0.612 & 0.645 & 13.775 & -- \\
& Full LA                   & 0.263 & 0.325 & 49.158 & -- \\
& Full LA + pruning         & 0.276 & \textbf{0.317} & 36.293 & -- \\
& Smart LA1                 & 0.411 & 0.462 & 37.550 & 0.702 \\
& Smart LA1 + pruning       & 0.431 & 0.457 & 21.845 & 0.702 \\
& Smart LA2                 & 0.289 & 0.351 & 50.098 & 0.924 \\
& Smart LA2 + pruning       & 0.316 & 0.336 & 20.775 & 0.924 \\
\hline
3 & Conventional            & 0.420 & 0.504 & 36.678 & -- \\
& Conventional + pruning    & 0.464 & 0.483 & 9.258 & -- \\
& Full LA                   & 0.228 & 0.302 & 49.308 & -- \\
& Full LA + pruning         & 0.262 & \textbf{0.284} & 18.340 & -- \\
& Smart LA1                 & 0.262 & 0.302 & 37.153 & 0.651 \\
& Smart LA1 + pruning       & 0.286 & 0.295 & 8.675 & 0.651 \\
& Smart LA2                 & 0.256 & 0.300 & 38.133 & 0.892 \\
& Smart LA2 + pruning       & 0.283 & 0.292 & 9.208 & 0.892 \\
\hline
4 & Conventional            & 0.230 & 0.284 & 40.098 & -- \\
& Conventional + pruning    & 0.259 & \textbf{0.275} & 7.113 & -- \\
& Full LA                   & 0.214 & 0.303 & 52.578 & -- \\
& Full LA + pruning         & 0.269 & 0.277 & 5.335 & -- \\
& Smart LA1                 & 0.234 & 0.284 & 40.013 & 0.863 \\
& Smart LA1 + pruning       & 0.263 & 0.275 & 6.428 & 0.863 \\
& Smart LA2                 & 0.229 & 0.289 & 44.890 & 0.958 \\
& Smart LA2 + pruning       & 0.265 & 0.276 & 6.665 & 0.958 \\
\hline
\end{tabular}
}
\end{table}

\clearpage
% Tree Diagrams for Parkinsons Telemonitoring
\begin{figure}[htbp]
\centering
\setlength{\abovecaptionskip}{2pt}
\makebox[\textwidth][c]{\includegraphicsorplaceholder[width=1.2\textwidth]{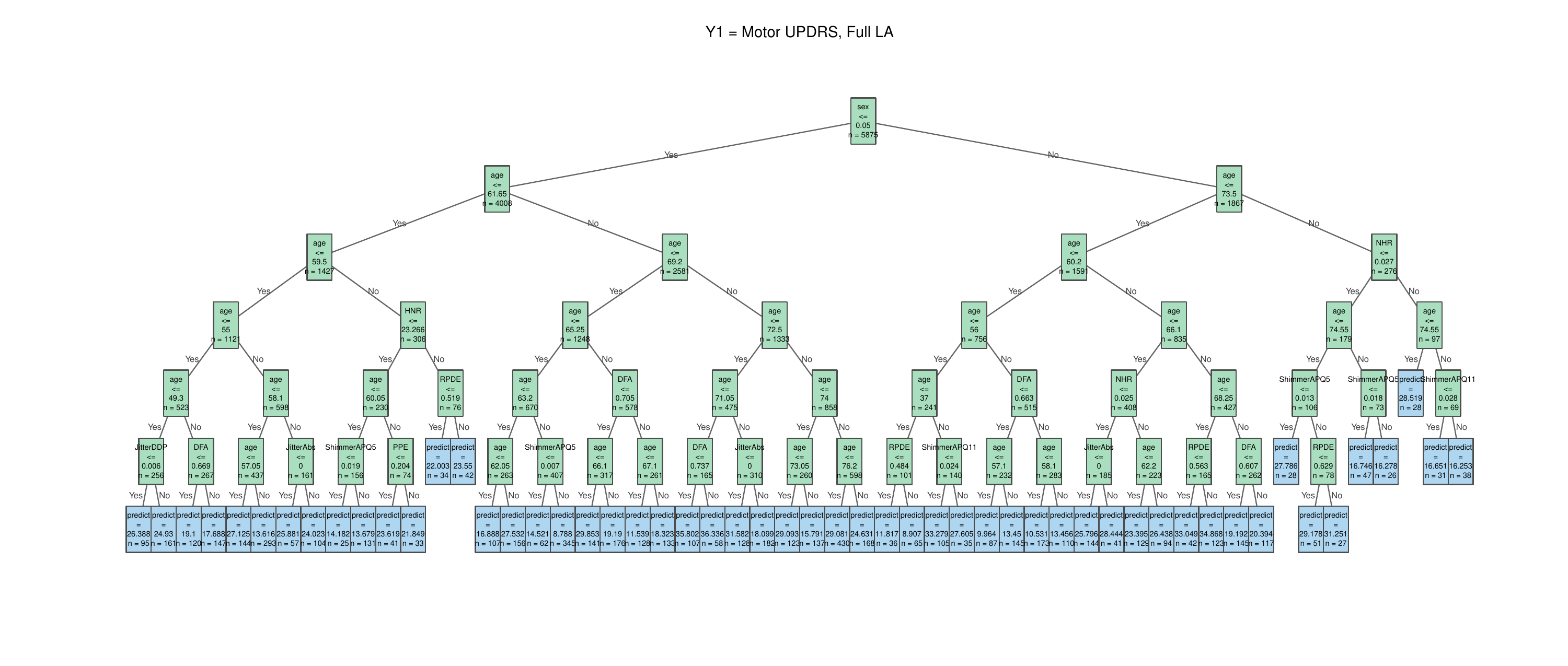}}
{
\caption{Final regression tree from the full look-ahead method for motor UPDRS using all observations without pruning.}
\label{fig:parkinsons-y1-full-la-tree}
}
\end{figure}

\begin{figure}[H]
\centering
\setlength{\abovecaptionskip}{2pt}
\makebox[\textwidth][c]{\includegraphicsorplaceholder[width=1.2\textwidth]{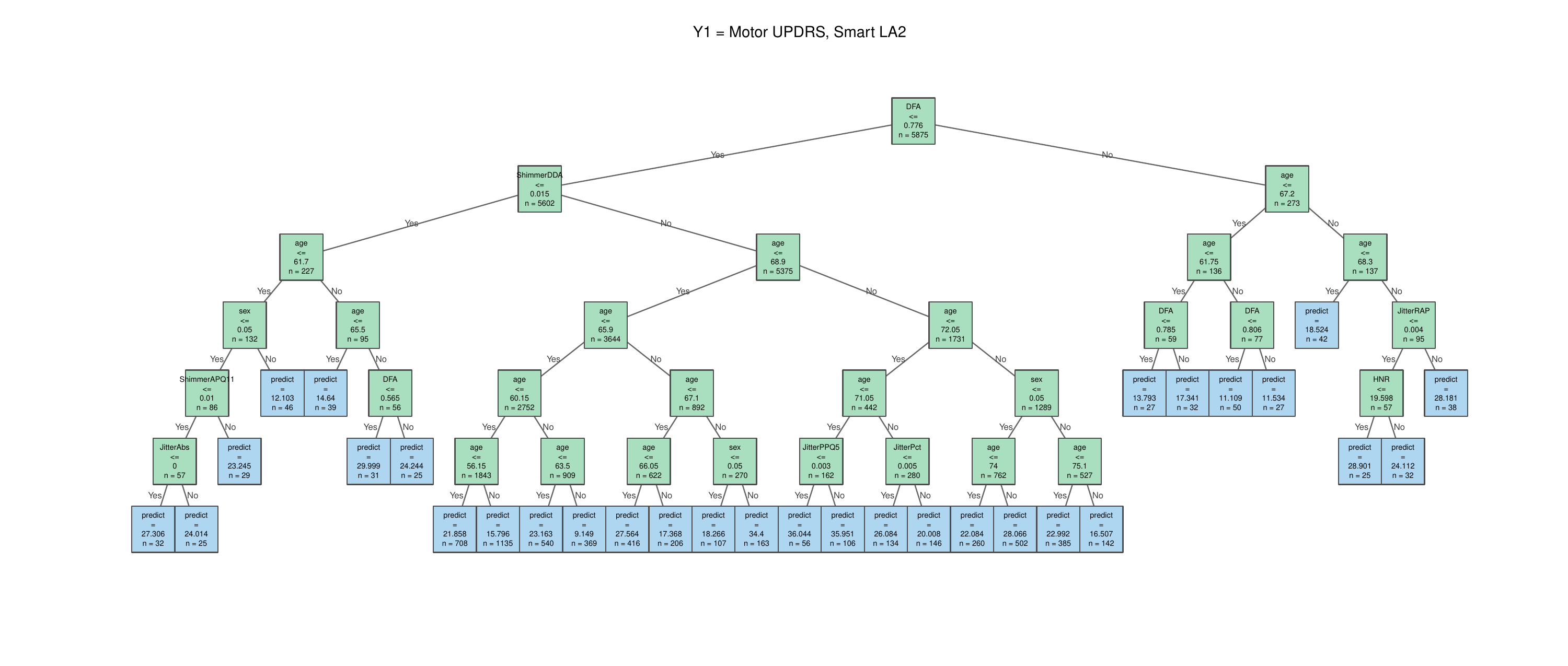}}
{
\caption{Final regression tree from the second smart look-ahead method for motor UPDRS using all observations without pruning.}
\label{fig:parkinsons-y1-smart-la2-tree}
}
\end{figure}
\clearpage

\begin{figure}[H]
\centering
\setlength{\abovecaptionskip}{2pt}
\makebox[\textwidth][c]{\includegraphicsorplaceholder[width=1.2\textwidth]{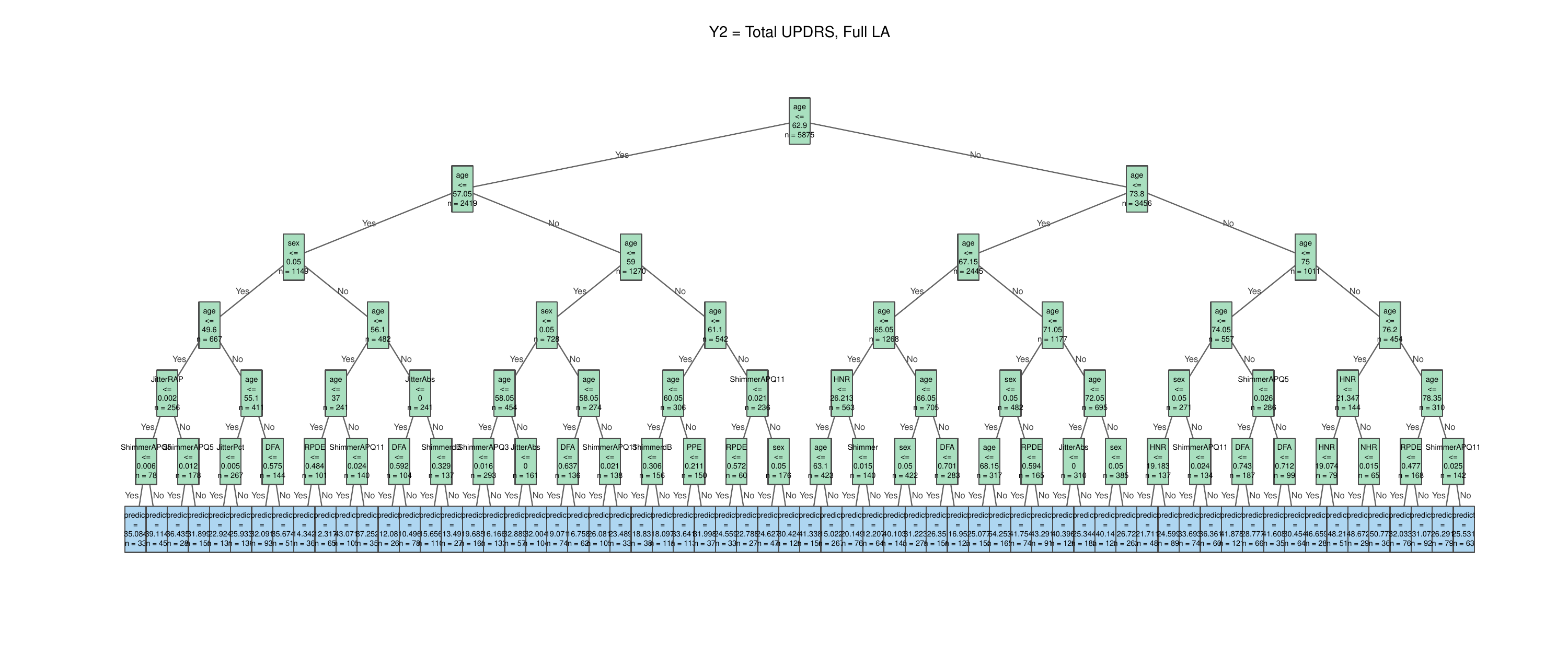}}
{
\caption{Final regression tree from the full look-ahead method for total UPDRS using all observations without pruning.}
\label{fig:parkinsons-y2-full-la-tree}
}
\end{figure}

\begin{figure}[H]
\centering
\setlength{\abovecaptionskip}{2pt}
\makebox[\textwidth][c]{\includegraphicsorplaceholder[width=1.2\textwidth]{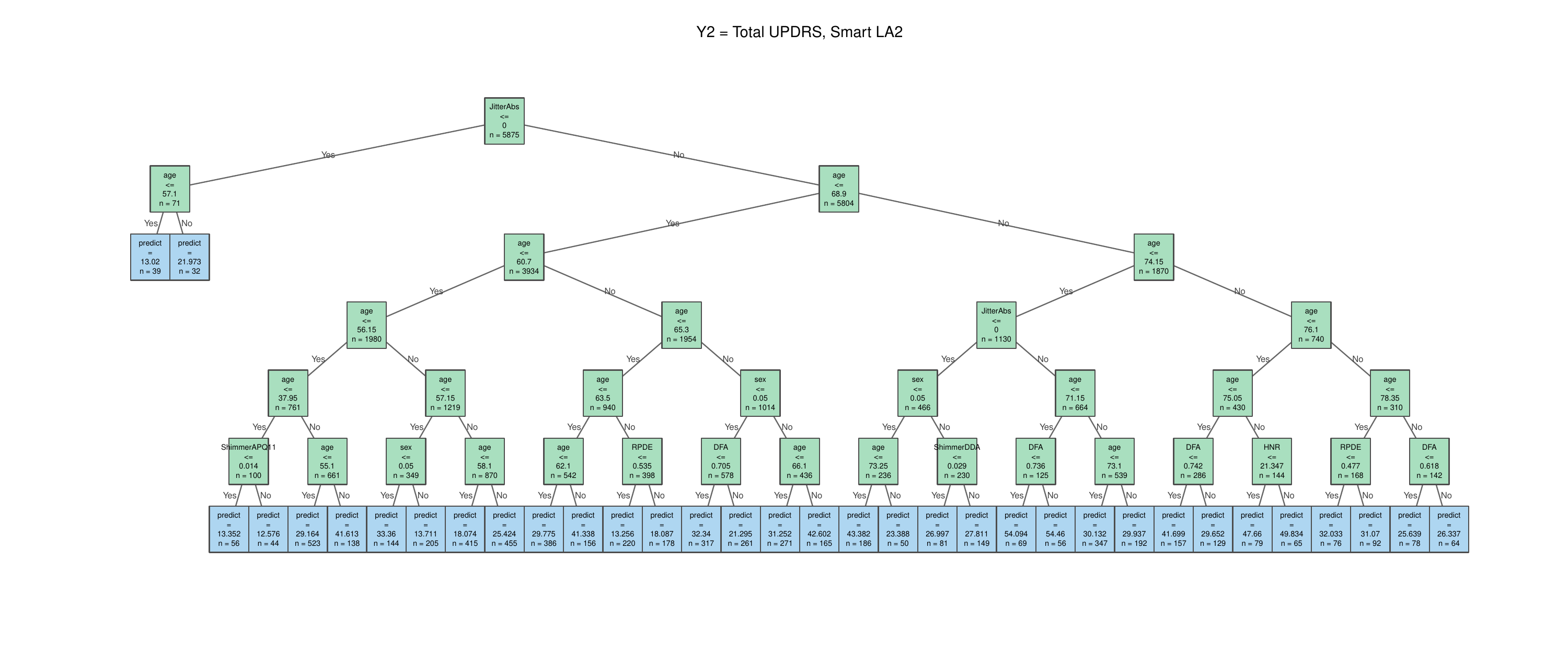}}
{
\caption{Final regression tree from the second smart look-ahead method for total UPDRS using all observations without pruning.}
\label{fig:parkinsons-y2-smart-la2-tree}
}
\end{figure}

\clearpage
% Tree Diagrams for Concrete Compressive Strength
\begin{figure}[htbp]
\centering
\setlength{\abovecaptionskip}{2pt}
\makebox[\textwidth][c]{\includegraphicsorplaceholder[width=1.4\textwidth]{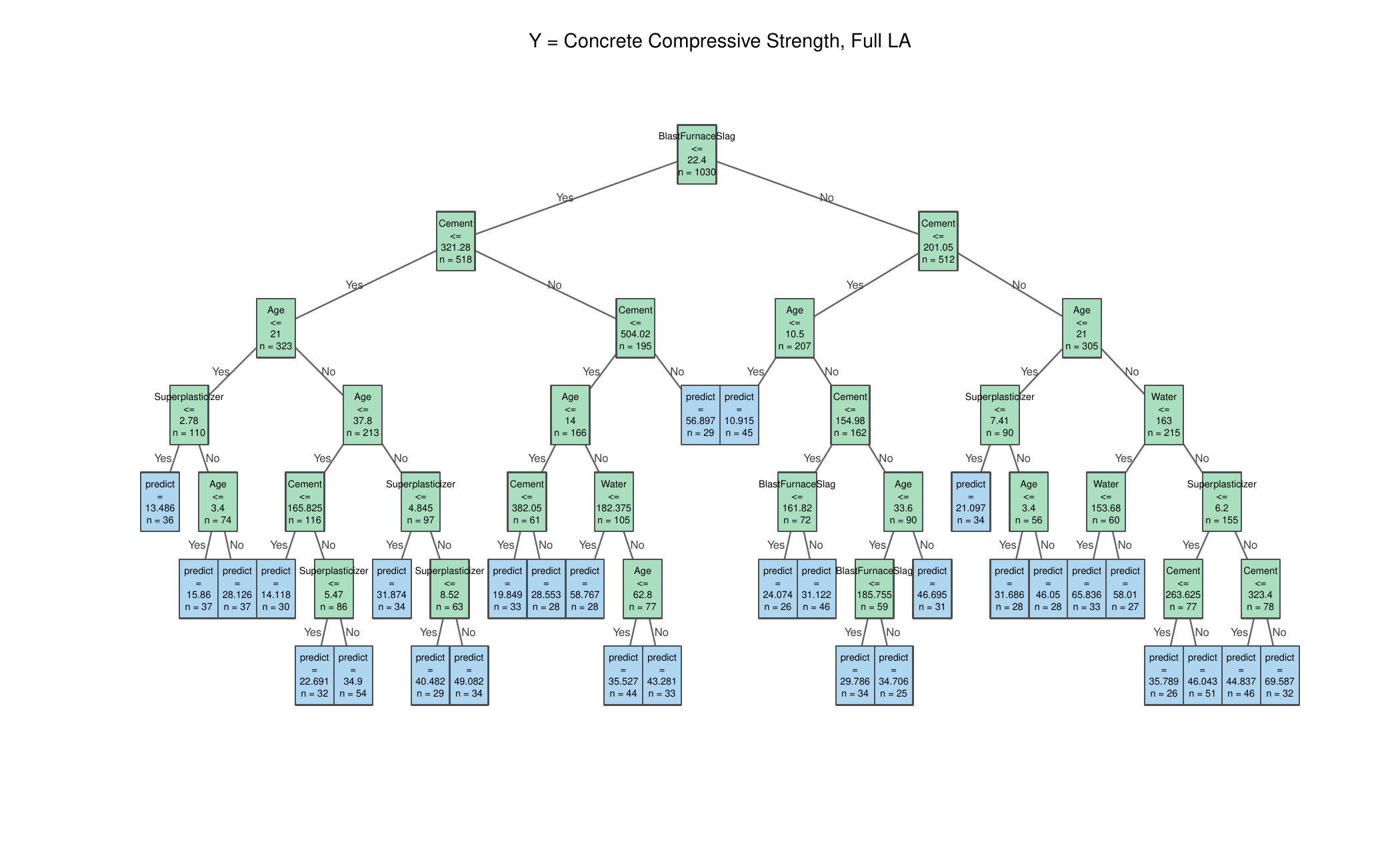}}
{
\caption{Final regression tree from the full look-ahead method for Concrete Compressive Strength using all observations without pruning.}
\label{fig:concrete-full-la-tree}
}
\end{figure}

\clearpage
\begin{figure}[htbp]
\centering
\setlength{\abovecaptionskip}{2pt}
\makebox[\textwidth][c]{\includegraphicsorplaceholder[width=1.4\textwidth]{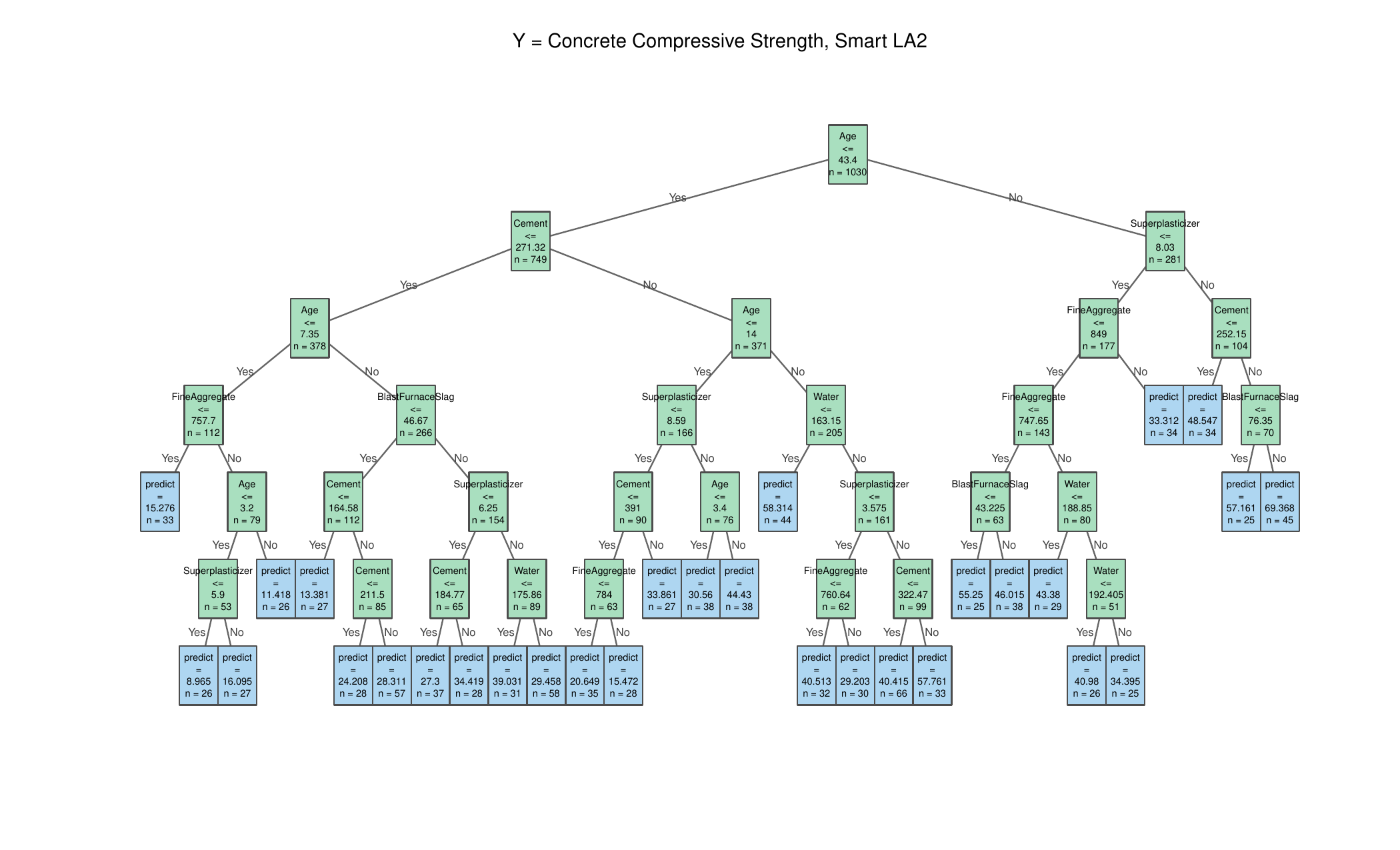}}
{
\caption{Final regression tree from the second smart look-ahead method for Concrete Compressive Strength using all observations without pruning.}
\label{fig:concrete-smart-la2-tree}
}
\end{figure}
\clearpage

\end{document}